%% file: paper.tex
\documentclass{article}

\PassOptionsToPackage{numbers, compress}{natbib}
\usepackage[final]{nips_2016}

\usepackage[utf8]{inputenc} % allow utf-8 input
\usepackage{booktabs}       % professional-quality tables
\usepackage{amsfonts}       % blackboard math symbols
\usepackage{amsmath}        % blackboard math symbols
\usepackage{microtype}      % microtypography
\input{defs}

\title{Reward Augmented Maximum Likelihood\\
for Neural Structured Prediction}

\author{
\begin{tabular}{p{1.4cm}ccccccp{1.4cm}}
\multicolumn{2}{p{3.2cm}}{\centering Mohammad Norouzi} &
\multicolumn{2}{p{2.6cm}}{\centering Samy Bengio} & 
\multicolumn{2}{p{2.6cm}}{\centering Zhifeng Chen} &
\multicolumn{2}{p{3.2cm}}{\centering Navdeep Jaitly}\\[.1cm]
& \multicolumn{2}{p{2.8cm}}{\centering Mike Schuster} &
\multicolumn{2}{p{2.8cm}}{\centering Yonghui Wu} & 
\multicolumn{2}{p{2.8cm}}{\centering Dale Schuurmans} &
\end{tabular}\\[.4cm]
~\texttt{\{mnorouzi,\,bengio,\,zhifengc,\,ndjaitly\}@google.com}\\
~\texttt{\{schuster,\,yonghui,\,schuurmans\}@google.com}\\[.2cm]
Google Brain
}

\comment{
mnorouzi,\,bengio,\,zhifengc,\,ndjaitly
schuster,\,yonghui,\,schuurmans
}

\begin{document}
% \nipsfinalcopy is no longer used

\maketitle

\input{abs}

\input{intro}

\input{method}
\input{new_bregman}

\input{related}
\input{expts}

\input{concl}
\input{ack}

%\bibliographystyle{splncs}
%\bibliographystyle{plain}
%\vspace*{-.4cm}
\bibliographystyle{abbrv}
\small{
  \bibliography{bib}%
}

\appendix
\input{new_appendix}

\end{document}

%% file: defs.tex
\usepackage{amsthm}
\usepackage{amsmath}
\usepackage{color}
\usepackage{pgfplots, pgfplotstable}
\newcommand{\mohammad}[1]{\textcolor{blue}{(Mo: #1)}}
\newcommand{\dale}[1]{\textcolor{green}{(Dale: #1)}}

\newcommand{\comment}[1]{}

\newcommand{\argmax}[1]{\underset{#1}{\operatorname{argmax}}}

\newcommand{\logsumexp}[0]{\mathrm lse}
\def\etal{{\em et al.}}

\newcommand{\extra}[1]{#1}
\newcommand{\extracite}[1]{#1}
\newcommand{\supp}[0]{Appendix~\ref{app:rml-proofs}}

\renewcommand{\vec}[1]{\boldsymbol{\mathbf{#1}}}
\def\btheta{\vec{\theta}}

\def\bx{\vec{x}}
\def\by{\vec{y}}

\def\calF{\mathcal{F}}
\def\calY{\mathcal{Y}}

\def\Real{\mathbb{R}}

\def\targetr{r}               % target logits
\def\targetp{q}               % target prob
\def\modelp{{p}_{\theta}} % model prob
\def\anyr{{s}}            % generic logits
\def\anyp{{p}}            % generic prob
\def\midr{\smallfrac{q+p}{2}}                  % midpoint logits
\def\altr{a}                  % alternate logits
\def\altaltr{b}               % alternate2 logits
\def\domain{\calF}            % bregman domain
\def\range{\calF^*}           % bregman range
\def\softmax{f^*}

\def\eps{\epsilon}

\def\bys{{\mathbf y}^*}
\def\byp{{\mathbf y}} % \widetilde

\def\temp{\tau}
\def\lossml{\mathcal{L}_{\mathrm{ML}}}
\def\lossrml{\mathcal{L}_{\mathrm{RAML}}}

\def\lossrl{\mathcal{L}_{\mathrm{RL}}}
\def\dataset{\mathcal{D}}

\def\diag{\mathrm{Diag}}
\def\int{\mathrm{int}}
\DeclareMathOperator{\E}{\mathbb{E}}

\newcommand{\reward}[2]{r(#1, #2)}

\newcommand{\kl}[2]{D_{\mathrm{KL}}\left(#1~\Vert~#2\right)}
\newcommand{\dv}[3]{D_{#1}\left(#2~\Vert~#3\right)}
\newcommand{\df}[2]{\dv{F}{#1}{#2}}
\newcommand{\dfs}[2]{\dv{F^*}{#1}{#2}}
\newcommand{\dfst}[2]{\dv{F^*_\temp}{#1}{#2}}
\newcommand{\ent}[1]{\mathbb{H}\left(#1\right)}
\newcommand{\inner}[2]{\langle#1,#2\rangle}

\newcommand{\trans}[1]{{#1}^{\ensuremath{\mathsf{T}}}}

\newcommand{\smallfrac}[2]{{\textstyle \frac{#1}{#2}}}

\def\eg{{\em e.g.}}
\def\ie{{\em i.e.}}

\newcounter{prop}
\newtheorem{proposition}[prop]{Proposition}

\newcommand{\tabref}[1]{Table~\ref{#1}}
\newcommand{\figref}[1]{Figure~\ref{#1}}

%% file: abs.tex
\begin{abstract}

A key problem in structured output prediction is direct optimization of the task reward function that matters for test evaluation. This paper presents a simple and computationally efficient approach to incorporate task reward into a maximum likelihood framework. By establishing a link between the log-likelihood and expected reward objectives, we show that an optimal regularized expected reward is achieved when the conditional distribution of the outputs given the inputs is proportional to their exponentiated scaled rewards. Accordingly, we present a framework to smooth the predictive probability of the outputs using their corresponding rewards. We optimize the conditional log-probability of augmented outputs that are sampled proportionally to their exponentiated scaled rewards. Experiments on neural sequence to sequence models for speech recognition and machine translation show notable improvements over a maximum likelihood baseline by using reward augmented maximum likelihood (RAML), where the rewards are defined as the negative edit distance between the outputs and the ground truth labels.

\end{abstract}

%% file: intro.tex
\vspace{-.2cm}
\section{Introduction}
\vspace{-.2cm}

Structured output prediction is ubiquitous in machine learning.
Recent advances in natural language processing, machine translation,
and speech recognition hinge on the development of better
discriminative models for structured outputs and sequences.  The
foundations of learning structured output models were established by
the seminal work on conditional random fields
(CRFs)~\cite{laffertyetal01} \extra{and variants~\cite{lecunetal98},}
and structured large margin
methods~\cite{\extracite{tsochantaridisetal05,}taskar2004}, which
demonstrate how generalization performance can be significantly
improved when one considers the joint effects of the predictions
across multiple output components. These models have evolved into
their deep neural
counterparts~\cite{sutskeveretal14,andor2016globally} through the use
of recurrent neural networks (RNN) with LSTM~\cite{hochreiter1997long}
\extra{and GRU~\cite{cho2014learning} }cells and attention
mechanisms~\cite{bahdanau2014neural}.

A key problem in structured output prediction has always been to
enable direct optimization of the task reward (loss) used for test
evaluation. For example, in machine translation one seeks better BLEU
scores, and in speech recognition better word error rates. Not
surprisingly, almost all task reward metrics are not differentiable,
hence hard to optimize. Neural sequence
models~(\eg~\cite{sutskeveretal14,bahdanau2014neural}) optimize
conditional log-likelihood, \ie~the conditional log-probability of the
ground truth outputs given corresponding inputs.  These models do not
explicitly consider the task reward during training, hoping that
conditional log-likelihood serves as a good surrogate for the task
reward. Such methods make no distinction between alternative incorrect
outputs: log-probability is only measured on the ground truth
input-output pairs, and all alternative outputs are equally penalized
through normalization, whether near or far from the ground truth
target. We believe one can improve upon maximum likelihood~(ML)
sequence models if the difference in the rewards of alternative
outputs is taken into account.

\comment{Some previous work also suggests
so~\cite{taskar2004,ranzatoetal15,hintonetal15,lopezpazetal16}.}

Standard ML training, despite its limitations, 
has enabled the training of deep RNN models,
leading to revolutionary advances in machine translation
\cite{sutskeveretal14,bahdanau2014neural,luong2015effective} and
speech
recognition~\cite{chan2015listen,chorowski2014end,chorowski15\comment{,dario}}.
A key property of ML training for locally normalized RNN models is
that the objective function factorizes into individual loss terms,
which could be {\em efficiently} optimized using stochastic gradient
descend~(SGD). This training procedure does not require any form of
inference or sampling from the model during training, leading to
computational efficiency and ease to implementation. By contrast,
almost all alternative formulations for training structure prediction
models require some form of inference or sampling from the model at
training time which slows down training, especially for deep RNNs
(\eg~see large margin, search-based~\cite{daumeetal09,
wiseman2016sequence}\comment{,bahdanauetal15}, and expected risk
optimization methods).

\comment{
RNNs~(\eg~).  Such methods incorporate some task reward
approximation during training, but the behavior of the approximation
is not well understood, especially for deep neural nets. Moreover,
these methods require some form of inference at training time, which
slows training.

methods incorporate task loss (negative reward) during training by
exploiting a continuous upper bound on empirical loss. While the
tightness of the bound is questionable, especially for deep networks,
these}

\comment{
Despite alternative approaches,
It is surprising that such an approach has significantly improved the
state of the art, even when measured with respect to task specific
rewards, since the training criterion decomposes over individual
components and 

no distinction is made between alternative incorrect
structures: likelihood is only measured on the target training
structure, and all alternative structures are penalized equally,
whether near or far from the target.  This appears to be a
fundamentally limited approach, and recent work has also shown that
more information can be gained by considering the differing rewards
between alternatives than by just considering a single best target
\cite{hintonetal15,lopezpazetal16\extracite{,vapnikizmailov15}}.
}

Our work is inspired by the use of reinforcement learning (RL)
algorithms, such as policy gradient~\cite{williams92}, to optimize
expected task
reward~\cite{ranzatoetal15\extracite{,bahdanau2016actor}}.  Even
though expected task reward seems like a natural objective, direct
policy optimization faces significant challenges: unlike ML, a
stochastic gradient given a mini-batch of training examples is
extremely noisy and has a high variance; gradients need to be
estimated via sampling from the model, which is a non-stationary
distribution; the reward is often sparse in a high-dimensional output
space, which makes it difficult to find any high value predictions,
preventing learning from getting off the ground; and, finally,
maximizing reward does not explicitly consider the supervised labels,
which seems inefficient.  In fact, all previous attempts at direct
policy optimization for structured output prediction have started by
bootstrapping from a previously trained ML
solution \cite{ranzatoetal15\extracite{,bahdanau2016actor},silveretal16},
using several heuristics and tricks to make learning stable.

\comment{
but they only consider a convex upper bound on an otherwise
non-convex task loss, while also requiring that loss-augmented
inference be performed during training.

In fact, explicitly considering task reward has been a central
challenge of structured output prediction from its beginning.
Motivated by the limitations of maximum likelihood training,
there have been a number of recent attempts 
to directly maximize expected task reward
by applying reinforcement learning techniques. 
}

This paper presents a new approach to task reward optimization that
combines the computational efficiency and simplicity of ML with the
conceptual advantages of expected reward maximization. Our algorithm
called {\em reward augmented maximum likelihood (RAML)} simply adds a
sampling step on top of the typical likelihood objective. Instead of
optimizing conditional log-likelihood on training input-output pairs,
given each training input, we first sample an output proportionally to
its exponentiated scaled reward. Then, we optimize log-likelihood on
such auxiliary output samples given corresponding inputs. When the
reward for an output is defined as its similarity to a ground truth
output, then the output sampling distribution is peaked at the
ground truth output, and its concentration is controlled by a
temperature hyper-parameter.

\comment{
As we discuss
below, our learning algorithm is tightly related to policy gradient
methods (\eg~\cite{williams92}), but it requires sampling from a
stationary distribution {\em vs.} model distribution.
}

Our theoretical analysis shows that the RAML and regularized expected
reward objectives optimize a KL divergence between the exponentiated
reward and model distributions, but in opposite directions. Further, we
show that at non-zero temperatures, the gap between the two criteria
can be expressed by a difference of variances measured on
interpolating distributions.  This observation reveals how entropy
regularized expected reward can be estimated by sampling from
exponentiated scaled rewards, rather than sampling from the model
distribution. \comment{, which is the basis for our algorithm}

\comment{
The benefit of this approach is that we retain a principled connection
to maximizing expected reward by training the model with 
standard supervised learning updates.

The training output structures are only sampled from a distribution that
is focused around the target output structure,
bypassing the intrinsic hardness of performing reward augmented inference
to find optimal output structures under a non-decomposable task reward.
}

Remarkably, we find that the RAML approach achieves significantly
improved results over state of the art maximum likelihood RNNs. We
show consistent improvement on both speech recognition (TIMIT dataset)
and machine translation (WMT'14 dataset), where output sequences are
sampled according to their edit distance to the ground truth
outputs. Surprisingly, we find that the best performance is achieved
with output sampling distributions that shift a lot of the weight away
from the ground truth outputs. In fact, in our experiments, the
training algorithm rarely sees the original unperturbed outputs.  Our
results give further evidence that models trained with imperfect
outputs and their reward values can improve upon models that are only
exposed to a single ground truth output per
input \cite{hintonetal15,lopezpazetal16\extracite{,vapnikizmailov15}}.

\comment{
a more focused
distribution around the supervised target defined by exponentiated
scaled rewards.  

We exploit this connection to develop a new training
method where training structures are drawn by randomly editing a given
target structure, weighted by their exponentiated scaled reward.  

The
parameters of a deep RNN model can then be updated by optimizing
weighted likelihood on the alternative structures.
}

%% file: method.tex
\vspace*{-.1cm}
\section{Reward augmented maximum likelihood}
\vspace*{-.2cm}
\label{sec:method}

Given a dataset of input-output pairs, $\dataset \equiv \{(\bx^{(i)},
\by^{*(i)})\}_{i=1}^N$, structured output models learn a parametric score
function $\modelp(\byp \mid \bx)$, which scores different output
hypotheses, $\byp \in \calY$.  We assume that the set of possible
output, $\calY$ is finite, \eg~English sentences up to a maximum
length. In a probabilistic model, the score function is normalized,
while in a large-margin model the score may not be normalized. In
either case, once the score function is learned, given an input $\bx$,
the model predicts an output $\widehat{\by}$ achieving maximal score,
\begin{equation}
\widehat{\by}(\bx) = \argmax{\byp}
~\modelp(\byp \mid \bx)~.
\end{equation}
If this optimization is intractable, approximate inference (\eg~beam
search) is used. We use a reward function $\reward{\by}{\bys}$ to
evaluate different proposed outputs against ground-truth outputs. Given a test
dataset $\dataset'$, one computes $\sum_{(\bx,\bys) \in \dataset'}
\reward{\widehat{\by}(\bx)}{\bys}$ as a measure of empirical reward. 
Since models with larger empirical reward are preferred, ideally one hopes
to maximize empirical reward during training.

However, since empirical reward is not amenable to numerical
optimization, one often considers optimizing alternative
differentiable objectives. 
The maximum likelihood (ML) framework tries to
minimize negative log-likelihood of the parameters given the data,
\begin{equation}
\lossml(\btheta; \dataset) = \sum_{(\bx,\bys) \in \dataset} -\log \modelp(\bys \mid \bx)~.
\label{eq:lossml}
\end{equation}
Minimizing this objective increases the conditional probability of
the target outputs, $\log \modelp(\bys \mid \bx)$, while decreasing
the conditional probability of alternative incorrect outputs. According
to this objective, all negative outputs are equally wrong, and
none is preferred over the others.

\comment{
ML training equally applies to models that
define $\modelp$ based on locally normalized distributions
(\eg~\cite{sutskeveretal14}) (\ie~directed graphical models), or
globally normalized models such as CRF (\ie~undirected graphical
models). Below, we will explain how we generalize ML to incorporate a
reward or loss function into account.
}

By contrast, reinforcement learning (RL) advocates optimizing
expected reward (with a maximum entropy regularizer~\cite{williams1991function\extracite{,mnih2016asynchronous}}), which is
formulated as minimization of the following objective,
\begin{equation}
\lossrl(\btheta; \temp, \dataset) = \sum_{(\bx,\bys) \in \dataset} \bigg\{
- \temp\ent{\modelp(\byp \mid \bx)} -\sum_{\by \in \calY} \modelp(\byp \mid
\bx)~\reward{\byp}{\bys} \bigg\},
\label{eq:lossrl}
\end{equation}
where $\reward{\byp}{\bys}$ denotes the reward function, \eg~negative
edit distance or BLEU score, $\temp$ controls the degree of
regularization, and $\ent{p}$ is the entropy of a distribution $p$,
\ie~$\ent{p(\byp)} = -\sum_{\byp \in \calY}p(\byp) \log p(\byp)$. It
is well-known that optimizing $\lossrl(\btheta; \temp)$ using SGD is
challenging because of the large variance of the gradients. Below we
describe how ML and RL objectives are related, and propose a hybrid
between the two that combines their benefits for supervised learning.

Let us define a distribution in the output space, termed the {\em
  exponentiated payoff distribution}, that is central in linking ML
  and RL objectives:
\begin{equation}
\targetp(\byp \mid \bys; \temp) = \frac{1}{Z(\bys, \temp)} \exp{ \left\{
  \reward{\byp}{\bys} /\temp \right\} }~,
\label{eq:payoff}
\end{equation}
where $Z(\bys, \temp) = \sum_{\byp \in \calY} \exp{ \left\{
  \reward{\byp}{\bys} /\temp \right\} }$. One can verify that the
global minimum of $\lossrl(\btheta; \temp)$, \ie~the optimal regularized
expected reward, is achieved when the model distribution matches
the exponentiated payoff distribution, \ie~$\modelp(\byp
\mid \bx) = \targetp(\byp \mid \bys; \temp)$. To see this, we
re-express the objective function in \eqref{eq:lossrl} in terms of a
KL divergence between $\modelp(\byp \mid \bx)$ and $\targetp(\byp \mid
\bys; \temp)$,
\begin{equation}
\sum_{(\bx,\bys) \in \dataset} \kl{\modelp(\byp \mid \bx)}{\targetp(\byp \mid
  \bys; \temp)} = \frac{1}{\,\temp\,}\lossrl(\btheta; \temp) + \mathrm{constant}~,
\label{eq:RL-kl} 
\end{equation}
where the $\mathrm{constant}$ on the RHS is $\sum_{(\bx,\bys) \in
  \dataset} \log Z(\bys, \temp)$. Thus, the minimum of
$\kl{\modelp}{\targetp}$ and $\lossrl$ is achieved when $\modelp =
\targetp$. At $\temp = 0$, when there is no entropy regularization,
the optimal $\modelp$ is a delta distribution, $\modelp(\byp \mid
\bx) = \delta(\byp \mid \bys)$, where $\delta(\byp \mid \bys) = 1$ at
$\byp = \bys$ and $0$ at $\byp \ne \bys$. Note that $\delta(\byp \mid
\bys)$ is equivalent to the exponentiated payoff distribution in the limit as $\temp \to
0$.

Returning to the log-likelihood objective, one can verify
that \eqref{eq:lossml} is equivalent to a KL divergence in the
opposite direction between a delta distribution
$\delta(\byp \mid \bys)$ and the model distribution
$\modelp(\byp \mid \bx)$,
\begin{equation}\label{eq:rlkl}
\sum_{(\bx,\bys) \in \dataset}\kl{\delta(\byp \mid \bys)}{\modelp(\byp
  \mid \bx)} = \lossml(\btheta)~.
\end{equation}
There is no constant on the RHS, as the entropy of a delta
distribution is zero, \ie~$\ent{\delta(\byp \mid \bys)} = 0$.

We propose a method called {\em reward-augmented maximum likelihood
  (RAML)}, which generalizes ML by allowing a non-zero temperature
parameter in the exponentiated payoff distribution, while still
optimizing the KL divergence in the ML direction. The RAML objective
function takes the form,
\begin{equation}
\lossrml(\btheta; \temp, \dataset) = \sum_{(\bx,\bys) \in \dataset} \bigg\{ - \sum_{\byp \in \calY} \targetp(\byp \mid \bys; \temp) \log \modelp(\by \mid \bx) \bigg\}~,
\label{eq:lossrml}
\end{equation}
which can be re-expressed in terms of a KL divergence as follows,
\begin{equation}
\sum_{(\bx,\bys) \in \dataset}\kl{\targetp(\byp \mid \bys;
  \temp)}{\modelp(\byp \mid \bx)} = \lossrml(\btheta; \temp) +
\mathrm{constant}~,
\label{eq:RAML-kl}
\end{equation}
where the $\mathrm{constant}$ is
$-\sum_{(\bx,\bys) \in \dataset}\ent{q(\byp \mid \bys, \temp)}$.  Note
that the temperature parameter, $\temp \ge 0$, serves as a
hyper-parameter that controls the smoothness of the optimal
distribution around correct targets by taking into account the reward
function in the output space. The objective functions
$\lossrl(\btheta; \temp)$ and $\lossrml(\btheta; \temp)$, have the
same global optimum of $\modelp$, but they optimize a KL divergence in
opposite directions. We characterize the difference between these
two objectives below, showing that they are equivalent up to their first
order Taylor approximations. For optimization convenience, we focus on
minimizing $\lossrml(\btheta; \temp)$ to achieve a good solution for
$\lossrl(\btheta; \temp)$.

\comment{,in \eqref{eq:RL-kl} and \eqref{eq:RAML-kl}}

\vspace{-.1cm}
\subsection{Optimization}
\vspace{-.1cm}

Optimizing the reward augmented maximum likelihood (RAML) objective,
$\lossrml(\btheta; \temp)$, is straightforward if one can draw
unbiased samples from $\targetp(\byp \mid \bys; \temp)$. We can
express the gradient of $\lossrml$ in terms of an expectation over
samples from $\targetp(\byp \mid \bys; \temp)$,
\begin{equation}
\nabla_{\btheta} \lossrml(\btheta; \temp) ~=~ \E_{\targetp(\byp \mid
  \bys; \temp)} \big[ -\nabla_{\btheta}\log \modelp(\by \mid
\bx) \big]~.
\label{eq:grad-rml}
\end{equation}
Thus, to estimate $\nabla_{\btheta} \lossrml(\btheta; \temp)$ given a
mini-batch of examples for SGD, one draws $\byp$ samples given
mini-batch $\bys$'s and then optimizes log-likelihood on such samples
by following the mean gradient. At a temperature $\temp = 0$, this
reduces to always sampling $\bys$, hence ML training with no
sampling.

By contrast, the gradient of $\lossrl(\btheta; \temp)$, based on
likelihood ratio methods, takes the form,
\begin{equation}
\nabla_{\btheta} \lossrl(\btheta; \temp) ~=~ \E_{\modelp(\byp \mid
  \bx)} \big[ -\nabla_{\btheta}\log \modelp(\by \mid
\bx) \cdot \reward{\byp}{\bys} \big]~.
\label{eq:grad-rl}
\end{equation}
There are several critical differences between \eqref{eq:grad-rml} and
\eqref{eq:grad-rl} that make SGD optimization of $\lossrml(\btheta;
\temp)$ more desirable. First, in \eqref{eq:grad-rml}, one has to sample
from a stationary distribution, the so called exponentiated payoff
distribution, whereas in \eqref{eq:grad-rl} one has to sample from the
model distribution as it is evolving. 
Not only does sampling from the model
potentially slow down training, one also needs to employ several tricks
to get a better estimate of the gradient of
$\lossrl$~\cite{ranzatoetal15}. A body of literature in reinforcement
learning focuses on reducing the variance of \eqref{eq:grad-rl} by
using sophisticated techniques such as {\em actor-critique}
methods~\cite{suttonbarto98, degris2012model}. Further, the reward is
often sparse in a high-dimensional output space, which makes finding
any reasonable prediction challenging when \eqref{eq:grad-rl} is
used to refine a randomly initialized model. Thus, smart model
initialization is needed. By contrast, we initialize the models
randomly and refine them using \eqref{eq:grad-rml}.

\vspace{-.1cm}
\subsection{Sampling from the exponentiated payoff distribution}
\vspace{-.1cm}

To compute the gradient of the model using the RAML approach, one
needs to sample auxiliary outputs from the exponentiated payoff
distribution, $\targetp(\byp \mid \bys; \temp)$. This sampling is the
price that we have to pay to learn with rewards. One should contrast
this with loss-augmented inference in structured large margin methods,
and sampling from the model \comment{and reward reweighting of the
gradients} in RL. We believe sampling outputs
proportional to exponentiated rewards is more efficient and effective
in many cases.

Experiments in this paper use reward values defined by either negative
Hamming distance or negative edit distance. We sample from
$\targetp(\byp \mid \bys; \temp)$ by stratified sampling, where we
first select a particular distance, and then sample an output with
that distance value. Here we focus on edit distance sampling, as
Hamming distance sampling is a simpler special case. Given a sentence
$\bys$ of length $m$, we count the number of sentences within an edit
distance $e$, where $e \in \{0, \ldots, 2m\}$. Then, we reweight the
counts by $\exp\{-e / \temp\}$ and normalize. Let $c(e, m)$ denote the
number of sentences at an edit distance $e$ from a sentence of length
$m$. First, note that a deletion can be thought as a substitution with
a nil token. This works out nicely because given a vocabulary of
length $v$, for each insertion we have $v$ options, and for each
substitution we have $v-1$ options, but including the nil token, there
are $v$ options for substitutions too. When $e = 1$, there are $m$
possible substitutions and $m+1$ insertions. Hence, in total there are
$(2m+1)v$ sentences at an edit distance of $1$. Note, that exact
computation of $c(e, m)$ is difficult if we consider all edge cases,
for example when there are repetitive words in $\bys$, but ignoring
such edge cases we can come up with approximate counts that are
reliable for sampling. When $e > 1$, we estimate $c(e, m)$ by
\begin{equation}
c(e, m) = \sum_{s=0}^m {m \choose s} {m + e - 2s \choose e - s} v^e~,
\end{equation}
where $s$ enumerates over the number of substitutions. Once $s$ tokens
are substituted, then those $s$ positions lose their significance, and
the insertions before and after such tokens could be merged. Hence,
given $s$ substitutions, there are really $m - s$ reference positions
for $e - s$ possible insertions. Finally, one can sample according to
BLEU score or other sequence metrics by importance sampling where the
proposal distribution could be edit distance sampling above.

%% file: new_bregman.tex
\vspace*{-.1cm}
\section{RAML analysis}
\vspace*{-.2cm}

In the RAML framework, we find the model parameters by minimizing the
objective \eqref{eq:lossrml} instead of optimizing the RL
objective, \ie~regularized expected reward in \eqref{eq:lossrl}. The
difference lies in minimizing
$\kl{\targetp(\byp \mid \bys; \temp)}{\modelp(\byp \mid \bx)}$ instead
of $\kl{\modelp(\byp \mid \bx)}{\targetp(\byp \mid \bys; \temp)}$. For
convenience, let's refer to $\targetp(\byp \mid \bys; \temp)$ as $q$,
and $\modelp(\byp
\mid \bx)$ as $p$. Here, we characterize the difference between the two
divergences, $\kl{q}{p} - \kl{p}{q}$, and use this analysis to
motivate the RAML approach.

We will initially consider the KL divergence in its more general form
as a Bregman divergence, which will make some of the key properties
clearer.  A Bregman divergence is defined by a strictly convex,
differentiable, closed potential function
$F:\domain\rightarrow\Real$~\cite{banerjeeetal05}. Given $F$ and two
points $p, q \in \domain$, the corresponding Bregman divergence
$D_F:\domain\times\domain\rightarrow\Real^+$ is defined by
\begin{equation}
\df{p}{q} = F\left(p\right) - F\left(q\right) - \trans{\left(p-q\right)} \nabla
F\left(q\right)~,
\end{equation}
the difference between the strictly convex potential at $p$ and its
first order Taylor approximation expanded about $q$.  Clearly this
definition is not symmetric between $p$ and $q$.  By the strict
convexity of $F$ it follows that $\df{p}{q} \geq 0$ with $\df{p}{q} =
0$ if and only if $p = q$. To characterize the difference between
opposite Bregman divergences, we provide a simple result that relates
the two directions %for an arbitrary Bregman divergence.  
under suitable conditions.
Let $H_F$ denote the Hessian of $F$.

%\vspace*{-.2cm}
\begin{proposition}
For any twice differentiable strictly convex closed potential $F$,
and $p, q\in\int(\domain)$:
%\vspace*{-1\baselineskip}
\begin{align}
\df{q}{p}
&=
\df{p}{q}
+
\smallfrac{1}{4}
\trans{(p-q)}
\big(H_{F}(\altr)-H_F(\altaltr)\big) 
(p-q)
\end{align}
for some $\altr=(1-\alpha)p+\alpha q$,~
($0\leq\alpha\leq \smallfrac{1}{2}$),~
$\altaltr=(1-\beta)q+\beta p$,~
($0\leq\beta\leq\smallfrac{1}{2}$).~~~(see \supp{})\label{propn:prelim}
\end{proposition}
\vspace*{-.1cm}
For probability vectors $p, q \in \Delta^{|\calY|}$ and a potential
$F\left(p\right) = -\temp\ent{p}$, $\df{p}{q} = \temp\kl{p}{q}$.  Let
$\softmax : \Real^{|\calY|} \to \Delta^{|\calY|}$ denote a normalized
exponential operator that takes a real-valued logit vector and turns
it into a probability vector. Let $r$ and $s$ denote real-valued
logit vectors such that $q = \softmax(r / \temp)$ and $p
= \softmax(s/\temp)$. Below, we characterize the gap between
$\kl{p(y)}{q(y)}$ and $\kl{q(y)}{p(y)}$ in terms of the difference
between $s(y)$ and $r(y)$.

\comment{
Then, $q
= \softmax(r / \temp)$, where $r$ denodes a vector of rewards for all
$\by$. Similarly, suppose $s$ denotes the logits (log probabilities)
predicted by the model for different outputs divided by $\temp$, so
that $p = \softmax(s/\temp)$. One can characterize the gap between
$\kl{p}{q}$ and $\kl{q}{p}$ as,
}

%\vspace*{-.2cm}
\begin{proposition}
The KL divergence between $p$ and $q$ in two directions can be
expressed as,
\begin{equation*}
\begin{aligned}
%\begin{align}
\kl{p}{q}
%D_{F^*_\temp}(\anyp\|\targetp)
& =
\kl{q}{p}
%D_{F^*_\temp}(\targetp\|\anyp)
+
\smallfrac{1}{4\temp^2}\;
\mathrm{Var}_{y\sim \softmax(\altr / \temp)}\left[\anyr(y) - \targetr(y) \right]
-
\smallfrac{1}{4\temp^2}\;
\mathrm{Var}_{y\sim \softmax(\altaltr / \temp)}\left[\anyr(y) - \targetr(y) \right]
%\label{eq:equality}
\\
& <
\kl{q}{p}
%D_{F^*_\temp}(\targetp\|\anyp)
+
\smallfrac{1}{\temp^2}\;
\| \anyr-\targetr \|_2^2
,
%\end{align}
\end{aligned}
\label{eq:inequality}
\end{equation*}
for some 
$\altr=(1-\alpha)\anyr+\alpha\targetr$,~
($0\leq\alpha\leq \smallfrac{1}{2}$),~
$\altaltr=(1-\beta)\targetr+ \beta \anyr$,~
($0\leq\beta\leq \smallfrac{1}{2}$).~~~(see \supp{})
\vspace{-.3cm}
\label{propn:variance}
\end{proposition}
Given Proposition~\ref{propn:variance}, one can relate the two %RL and RAML
 objectives, $\lossrl(\btheta; \temp)$~\eqref{eq:RL-kl} and
 $\lossrml(\btheta; \temp)$~\eqref{eq:RAML-kl}, by
\begin{equation}
\lossrl
= \temp \lossrml +
\smallfrac{1}{4\temp} \!\!\sum_{(\bx,\bys) \in \dataset} \Big\{
\mathrm{Var}_{\by\sim \softmax(\altr / \temp)}\left[\anyr(\by) - \targetr(\by) \right]
-
\mathrm{Var}_{\by\sim \softmax(\altaltr / \temp)}\left[\anyr(\by) - \targetr(\by) \right]
\Big\}~,
\label{eq:rml-rl}
\end{equation}
where $\anyr(\by)$ denotes $\temp$-scaled logits predicted by the
model such that $\modelp(\byp \mid \bx) = \softmax( \anyr(\by)
/ \temp)$, and $\targetr(\by) = \reward{\byp}{\bys}$. The gap between
regularized expected reward \eqref{eq:RL-kl} and $\temp$-scaled RAML
criterion \eqref{eq:RAML-kl} is simply a difference of two variances,
whose magnitude decreases with increasing regularization.
Proposition~\ref{propn:variance} also shows an opportunity for
learning algorithms: if $\temp$ is chosen so that
$\targetp=\softmax(\targetr / \temp)$, then $\softmax(\altr / \temp)$
and $\softmax(\altaltr / \temp)$ have lower variance than $\anyp$
(which can always be achieved for sufficiently small $\temp$ provided
$\anyp$ is not deterministic), then the expected regularized reward
under $\anyp$, and its gradient for training, can be exactly
estimated, in principle, by including the extra variance terms and
sampling from more focused distributions than $\anyp$.  Although we
have not yet incorporated approximations to the additional variance
terms into RAML, this is an interesting research direction.

%% file: related.tex
\section{Related Work}
\vspace*{-.2cm}

The literature on structure output prediction is vast, falling into
three broad categories: (a) supervised learning approaches that ignore
task reward and use supervision; (b) reinforcement learning
approaches that use only task reward and ignore supervision; and (c)
hybrid approaches that attempt to exploit both supervision and task
reward. This paper clearly falls in category (c).

Work in category (a) includes classical conditional random
fields~\cite{laffertyetal01} and conditional log-likelihood training of
RNNs~\cite{sutskeveretal14,bahdanau2014neural}. It also includes the
approaches that attempt to perturb the training inputs and supervised
training structures to improves the robustness (and hopefully the
generalization) of the conditional models
(\eg~see~\cite{bengioetal15,kumaretal15\extracite{,strubelletal15}}). These
approaches offer improvements to standard maximum likelihood
estimation, but they are fundamentally limited by not incorporating
a task reward. \extra{The DAGGER method~\cite{rossetal10} also focuses
on using supervision only, but can be extended to use a task loss;
even then, the DAGGER assumes that an expert is available to label
every alternative sequence, which does not fit the usual structured
prediction scenario.}

By contrast, work in category (b) includes reinforcement learning
approaches that only consider task reward and do not use any other 
supervision.
Beyond the traditional reinforcement learning approaches,
such as policy gradient
\cite{williams92,suttonetal00},
and actor-critic
\cite{suttonbarto98},
Q-learning
\cite{vanhasseltetal15},
this category includes SEARN \cite{daumeetal09}.
There is some relationship to the work presented here and
work on relative entropy policy search
\cite{petersetal10},
and
policy optimization via expectation maximization
\cite{vlassisetal09}
and KL-divergence
\cite{kappen12,todorov06},
however none of these bridge the gap between the two 
directions of the KL-divergence, nor do they consider any
supervision data as we do here.

\comment{
Recently, this line of research has been
improved to directly consider task reward \cite{tamiretal10}, but is
limited to a perceptron like
}

There is also a substantial body of related work in category (c),
which considers how to exploit supervision information while training
with a task reward metric.  A canonical example is large margin
structured
prediction~\cite{taskar2004\extracite{,tsochantaridisetal05},gimpel2010softmax},
which explicitly uses supervision and considers an upper bound
surrogate for task loss. This approach requires loss augmented
inference that cannot be efficiently achieved for general task
losses. We are not aware of successful large-margin methods for neural
sequence prediction, but a related approach
by~\cite{wiseman2016sequence} for neural machine translation builds on
SEARN~\cite{daumeetal09}. Some form of inference during training is
still needed, and the characteristics of the objective are not well
studied. We also mentioned the work on maximizing task reward by
bootstrapping from a maximum likelihood policy
\cite{ranzatoetal15,silveretal16}, but such an approach only makes
limited use of supervision.  Some work in robotics has considered
exploiting supervision as a means to provide indirect sampling
guidance to improve policy search methods that maximize task reward
\cite{levinekoltun13a,levinekoltun13b,shen2015minimum\comment{,levinekoltun14,levineetal15}},
but these approaches do not make use of maximum likelihood training.
An interesting work is \cite{kimetal14} which explicitly 
incorporates supervision in the policy evaluation phase of a
policy iteration procedure that otherwise seeks to maximize task reward.
However, this approach only considers a greedy policy form
that does not lend itself to being represented as a deep RNN,
and has not been applied to structured output prediction.
Most relevant are ideas for improving approximate maximum likelihood training 
for intractable models by passing the gradient calculation through 
an approximate inference procedure \cite{domke12,stoyanovetal11}.
These works, however, are specialized to particular approximate inference
procedures, and, by directly targeting expected reward, are subject to the
variance problems that motivated this work.

One advantage of the RAML framework is its computational efficiency at
training time. By contrast, RL and scheduled
sampling~\cite{bengioetal15} require sampling from the model, which
can slow down the gradient computation by $2\times$. Structural SVM
requires loss-augmented inference which is often more expensive than
sampling from the model. Our framework only requires sampling from a
fixed exponentated payoff distribution, which can be thought as a form
of input pre-processing. This pre-processing can be parallelized by
model training by having a thread handling loading the data and
augmentation.

Recently, we were informed of the unpublished work of
Volkovs~\etal~\cite{volkovsetal11} that also proposes an objective
like RAML, albeit with a different derivation. No theoretical relation
was established to entropy regularized RL, nor was the method applied
to neural nets for sequences, but large gains were reported over
several baselines applying the technique to ranking problems with CRFs.

%% file: expts.tex
\vspace*{-.1cm}
\section{Experiments}
\vspace*{-.2cm}

We compare our approach, reward augmented maximum likelihood ({\em
  RAML}), with standard maximum likelihood ({\em ML}) training on
sequence prediction tasks using state-of-the-art attention-based
recurrent neural
networks~\cite{sutskeveretal14,bahdanau2014neural}. Our experiments
demonstrate that the RAML approach considerably outperforms ML baseline
on both speech recognition and machine translation tasks.

% On both tasks we consider edit distance a
% mechanism for samply target sentences.

\input{timit}

\input{wmt}

%% file: timit.tex
\vspace*{-.1cm}
\subsection{Speech recognition}
\vspace*{-.2cm}

For experiments on speech recognition, we use the TIMIT dataset; a
standard benchmark for clean phone recognition. This dataset consists of
recordings from different speakers reading ten phonetically rich
sentences covering major dialects of American English. We use the
standard train\,/\,dev\,/\,test splits suggested by the Kaldi
toolkit~\cite{povey2011kaldi}.

As the sequence prediction model, we use an attention-based
encoder-decoder recurrent model of~\cite{chan2015listen} with three
$256$-dimensional LSTM layers for encoding and one $256$-dimensional
LSTM layer for decoding. We do not modify the neural network
architecture or its gradient computation in any way, but we only
change the output targets fed into the network for gradient
computation and SGD update. The input to the network is a standard
sequence of $123$-dimensional log-mel filter response
statistics. Given each input, we generate new outputs around ground
truth targets by sampling according to the exponentiated payoff
distribution. We use negative edit distance as the measure of
reward. Our output augmentation process allows insertions, deletions,
and substitutions.

An important hyper-parameter in our framework is the temperature
parameter,~$\temp$, controlling the degree of output augmentation. We
investigate the impact of this hyper-parameter and report results for
$\temp$ selected from a candidate set of $\temp \in \{ 0.6,\, 0.65,\,
0.7,\, 0.75,\, 0.8,\, 0.85,\, 0.9,\, 0.95,\, 1.0\}$. At a temperature
of $\temp = 0$, outputs are not augmented at all, but as $\temp$
increases, more augmentation is
generated. \figref{fig:augmentation-prob} depicts the fraction of
different numbers of edits applied to a sequence of length $20$ for
different values of $\temp$. These edits typically include very small
number of deletions, and roughly equal number of insertions and
substitutions. For insertions and substitutions we uniformly sample
elements from a vocabulary of $61$ phones. According to
\figref{fig:augmentation-prob}, at $\temp = 0.6$, more than $60\%$ of
the outputs remain intact, while at $\temp = 0.9$, almost all target
outputs are being augmented with $5$ to $9$ edits being sampled with a
probability larger than $0.1$. We note that the augmentation becomes
more severe as the outputs get longer.

The phone error rates (PER) on both dev and test sets for different
values of $\temp$ and the ML baseline are reported in
\tabref{tab:per-on-timit}. Each model is trained and tested $4$ times,
using different random seeds. In \tabref{tab:per-on-timit}, we report
average PER across the runs, and in parenthesis the difference of
average error to minimum and maximum error. We observe that a
temperature of $\temp = 0.9$ provides the best results, outperforming
the ML baseline by $2.9\%$ PER on the dev set and $2.3\%$ PER on the
test set. The results consistently improve when the temperature
increases from $0.6$ to $0.9$, and they get worse beyond $\temp =
0.9$. It is surprising to us that not only the model trains with such
a large amount of augmentation at $\temp = 0.9$, but also it
significantly improves upon the baseline. Finally, we note that
previous work~\cite{chorowski2014end,chorowski15} suggests several
refinements to improve sequence to sequence models on TIMIT by adding
noise to the weights and using more focused forward-moving attention
mechanism. While these refinements are interesting and they could be
combined with the RAML framework, in this work, we do not
implement such refinements, and focus specifically on a fair
comparison between the ML baseline and the RAML method.

\definecolor{A}{HTML}{AA191C}
\definecolor{B}{HTML}{FDAE61}
\definecolor{C}{HTML}{ABDDA4}
\definecolor{D}{HTML}{2B83BA}
\definecolor{E}{HTML}{D2FF00}
\definecolor{F}{HTML}{FF00D2}
\definecolor{G}{HTML}{2FF00D}
\definecolor{H}{HTML}{00D2FF}
\definecolor{I}{HTML}{F63636}
\definecolor{J}{HTML}{00FFAA}
\definecolor{K}{HTML}{000000}

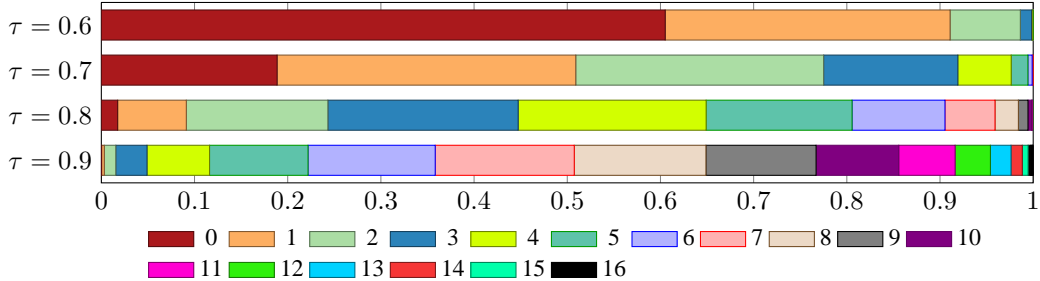
\begin{figure}
\pgfplotstableread[col sep=space, header=false]{
temp0.6  0.6051  0.3057  0.0754  0.0121  0.0014  0.0001  0.0000  0.0000  0.0000  0.0000  0.0000  0.0000  0.0000  0.0000  0.0000  0.0000  0.0000  0.0000  0.0000  0.0000  0.0000
temp0.7  0.1886  0.3205  0.2660  0.1440  0.0573  0.0179  0.0046  0.0010  0.0002  0.0000  0.0000  0.0000  0.0000  0.0000  0.0000  0.0000  0.0000  0.0000  0.0000  0.0000  0.0000
temp0.8  0.0175  0.0737  0.1519  0.2042  0.2017  0.1566  0.0997  0.0537  0.0251  0.0103  0.0038  0.0013  0.0004  0.0001  0.0000  0.0000  0.0000  0.0000  0.0000  0.0000  0.0000
temp0.9  0.0003  0.0029  0.0123  0.0335  0.0671  0.1057  0.1365  0.1492  0.1413  0.1182  0.0887  0.0605  0.0379  0.0221  0.0121  0.0062  0.0030  0.0014  0.0006  0.0003  0.0001
}\datatable
\begin{tikzpicture}
  \begin{axis}[
    width=\textwidth, %
    xbar stacked, %
    legend style={
        legend columns=11,
        at={(xticklabel cs:0.5)},
        anchor=north,
        draw=none,
        font=\footnotesize
    },
    y=-0.6cm,%
    bar width=0.4cm,%
    yticklabels = {$\temp = 0.6$, $\temp = 0.7$, $\temp = 0.8$, $\temp = 0.9$},%
    enlarge y limits={abs=0.5},%
    ytick=data,%
    xlabel={}, %
    xmin=0.0, xmax=1.0, %
    area legend
    ]%
\addplot[A!50!black,fill=A] table [x=1, y expr=\coordindex, text = 0] {\datatable};
\addplot[B!50!black,fill=B] table [x=2, y expr=\coordindex] {\datatable};
\addplot[C!50!black,fill=C] table [x=3, y expr=\coordindex] {\datatable};
\addplot[D!50!black,fill=D] table [x=4, y expr=\coordindex] {\datatable};
\addplot[E!50!black,fill=E] table [x=5, y expr=\coordindex] {\datatable};
\addplot[green!60!black,fill=cyan!30!green!60] table [x=6, y expr=\coordindex] {\datatable};
\addplot table [x=7, y expr=\coordindex] {\datatable};
\addplot table [x=8, y expr=\coordindex] {\datatable};
\addplot table [x=9, y expr=\coordindex] {\datatable};
\addplot table [x=10, y expr=\coordindex] {\datatable};
\addplot table [x=11, y expr=\coordindex] {\datatable};
\addplot[F!50!black,fill=F] table [x=12, y expr=\coordindex] {\datatable};
\addplot[G!50!black,fill=G] table [x=13, y expr=\coordindex] {\datatable};
\addplot[H!50!black,fill=H] table [x=14, y expr=\coordindex] {\datatable};
\addplot[I!50!black,fill=I] table [x=15, y expr=\coordindex] {\datatable};
\addplot[J!50!black,fill=J] table [x=16, y expr=\coordindex] {\datatable};
\addplot[K!50!black,fill=K] table [x=17, y expr=\coordindex] {\datatable};
\addplot[K!50!black,fill=K] table [x=18, y expr=\coordindex] {\datatable};
\addplot[K!50!black,fill=K] table [x=19, y expr=\coordindex] {\datatable};
\addplot[K!50!black,fill=K] table [x=20, y expr=\coordindex] {\datatable};
\legend{0,1,2,3,4,5,6,7,8,9,10,11,12,13,14,15,16}
\end{axis}
\end{tikzpicture}
\caption{ Fraction of different number of edits applied to a sequence
  of length $20$ for different $\temp$. At $\temp = 0.9$,
  augmentations with $5$ to $9$ edits are sampled with a probability $> 0.1$. [view in color]}
\label{fig:augmentation-prob}
\vspace*{-.1cm}
\end{figure}

\begin{table}
\begin{center}
\begin{tabular}{@{\hspace{.2cm}}c@{\hspace{.2cm}}|@{\hspace{.2cm}}c@{\hspace{.2cm}}|@{\hspace{.2cm}}c@{\hspace{.2cm}}}
%\begin{tabular}{ccccccc}
Method & Dev set & Test set\\
\hline
\hline
ML baseline	&	$20.87$ ($-0.2$, $+0.3$)  & $22.18$ ($-0.4$, $+0.2$) \\
\hline
\comment{
%LS, $\eps = 0.01$	&	$20.87$ ($-0.4$, $+0.5$)  &  $22.26$ ($-0.3$, $+0.2$)\\
%LS, $\eps = 0.03$	&	$20.67$ ($-0.2$, $+0.3$)  &  $23.21$ ($-1.1$, $+2.1$)\\ 
LS, $\eps = 0.05$	&	$20.71$ ($-0.6$, $+0.5$)  &  $21.99$ ($-0.3$, $+0.5$)\\
LS, $\eps = 0.10$	&	$19.74$ ($-0.2$, $+0.2$)  &  $21.37$ ($-0.4$, $+0.7$)\\
LS, $\eps = 0.20$	&	$19.69$ ($-0.8$, $+0.4$)  &  $20.89$ ($-0.2$, $+0.2$)\\
LS, $\eps = 0.50$       &       $19.09$ ($-0.3$, $+0.3$)  &  $20.82$ ($-0.8$, $+0.8$)\\
\hline
}
RAML, $\temp = 0.60$	&	$19.92$ ($-0.6$, $+0.3$)  & $21.65$ ($-0.5$, $+0.4$) \\
RAML, $\temp = 0.65$	&	$19.64$ ($-0.2$, $+0.5$)  & $21.28$ ($-0.6$, $+0.4$) \\
RAML, $\temp = 0.70$	&	$18.97$ ($-0.1$, $+0.1$)  & $21.28$ ($-0.5$, $+0.4$) \\
RAML, $\temp = 0.75$	&	$18.44$ ($-0.4$, $+0.4$)  & $20.15$ ($-0.4$, $+0.4$) \\
RAML, $\temp = 0.80$	&	$18.27$ ($-0.2$, $+0.1$)  & $19.97$ ($-0.1$, $+0.2$) \\
RAML, $\temp = 0.85$	&	$18.10$ ($-0.4$, $+0.3$)  & $19.97$ ($-0.3$, $+0.2$) \\
{\bf \boldmath  RAML, $\temp = 0.90$}	&	{\boldmath $18.00$} ($-0.4$, $+0.3$)  & {\boldmath $19.89$} ($-0.4$, $+0.7$) \\
RAML, $\temp = 0.95$	&	$18.46$ ($-0.1$, $+0.1$)  & $20.12$ ($-0.2$, $+0.1$) \\
RAML, $\temp = 1.00$	&	$18.78$ ($-0.6$, $+0.8$)  & $20.41$ ($-0.2$, $+0.5$) \\
\end{tabular}
\end{center}
\caption{ Phone error rates (PER) for different methods on TIMIT dev
  and test sets. Average PER of $4$ independent training runs is
  reported. }
\label{tab:per-on-timit}
\vspace*{-.3cm}
\end{table}

\vspace*{.3cm}

%% file: wmt.tex
\vspace*{-.3cm}
\subsection{Machine translation}
\vspace*{-.2cm}

We evaluate the effectiveness of the proposed approach on WMT'14
English to French machine translation benchmark. Translation quality
is assessed using {\em tokenized BLEU} score, to be consistent with
previous work on neural machine translation \cite{sutskeveretal14,
bahdanau2014neural, luong2015}. Models are trained on the full $36$M
sentence pairs from WMT'14 training set, and evaluated on $3003$
sentence pairs from newstest-2014 test set. To keep the sampling
process efficient and simple on such a large corpus, we augment
the output sentences only based on Hamming distance (\ie~edit distance
without insertion or deletion). For each sentece we sample a single
output at each step. One can consider insertions and deletions or
sampling according to exponentiated sentence BLEU scores, but we leave
that to future work.

\comment{ while BLEU score is primarily designed for translation
evaluation at a corpus level, not at an individual sentence
level.}

As the conditional sequence prediction model, we use an
attention-based encoder-decoder recurrent neural network similar
to~\cite{bahdanau2014neural}, but we use multi-layer encoder and
decoder networks consisting of three layers of $1024$ LSTM cells. As
suggested by~\cite{bahdanau2014neural}, for computing the softmax
attention vectors, we use a feedforward neural network with $1024$
hidden units, which operates on the last encoder and the first decoder
layers. In all of the experiments, we keep the network architecture
and the hyper-parameters fixed. All of the models achieve their peak
performance after about $4$ epochs of training, once we anneal the
learning rates. To reduce the noise in the BLEU score evaluation, we
report both peak BLEU score and BLEU score averaged among about $70$
evaluations of the model while doing the fifth epoch of training. We
perform beam search decoding with a beam size of $8$.

\tabref{tab:bleu-on-wmt} summarizes our experimental results on
WMT'14. We note that our ML translation baseline is quite strong, if
not the best among neural machine translation
models~\cite{sutskeveretal14, bahdanau2014neural, luong2015},
achieving very competitive performance for a single model. Even given
such a strong baseline, the RAML approach consistently improves the
results. Our best model with a temperature $\temp = 0.85$ improves
average BLEU by $0.4$, and best BLEU by $0.35$ points, which is a
considerable improvement. Again we observe that as we increase the
amount of augmentation from $\temp=0.75$ to $\temp=0.85$ the results
consistently get better, and then they start to get worse with more
augmentation.

\begin{table}
\begin{center}
\begin{tabular}{@{\hspace{.2cm}}c@{\hspace{.2cm}}|@{\hspace{.2cm}}c@{\hspace{.2cm}}|@{\hspace{.2cm}}c@{\hspace{.2cm}}}
Method & Average BLEU & Best BLEU\\
\hline
\hline
ML baseline	        & $36.50$ &  $36.87$ \comment{ $0.36866501$ } \\
\hline
RAML, $\temp = 0.75$	& $36.62$ &  $36.91$ \comment{ $36.914399$ } \\
RAML, $\temp = 0.80$	& $36.80$ &  $37.11$ \comment{ $37.114400$ } \\
{\bf \boldmath RAML, $\temp = 0.85$}	& {\boldmath $36.91$} &  {\boldmath $37.23$} \comment{ $37.233299$ } \\
RAML, $\temp = 0.90$	& $36.69$ &  $37.07$ \comment{ $37.073901$ } \\
RAML, $\temp = 0.95$	& $36.57$ &  $36.94$ \comment{ $36.936399$ } \\
\end{tabular}
\end{center}
\caption{ Tokenized BLEU score on WMT'14 English to French evaluated
  on newstest-2014 set. The RAML approach with different $\temp$
  considerably improves upon the maximum likelihood baseline. }
\label{tab:bleu-on-wmt}
\end{table}

{\bf Details.}  We train the models using asynchronous SGD with $12$
replicas without momentum. We use mini-batches of size $128$. We
initially use a learning rate of $0.5$, which we then exponentially
decay to $0.05$ after $800K$ steps. We keep evaluating the models
between $1.1$ and $1.3$ million steps and report average and peak BLEU
scores in \tabref{tab:bleu-on-wmt}. We use a vocabulary $200K$ words
for the source language and $80K$ for the target language. We only
consider training sentences that are up to $80$ tokens. We replace
rare words with several UNK tokens based on their first and last
characters.  At inference time, we replace UNK tokens in the output
sentences by copying source words according to largest attention
activations as suggested by~\cite{luong2015}.

%% file: concl.tex
\vspace*{-.1cm}
\section{Conclusion}
\vspace*{-.2cm}

We present a learning algorithm for structured output prediction that
generalizes maximum likelihood training by enabling direct
optimization of a task reward metric. Our method is computationally
efficient and simple to implement. It only requires augmentation of
the output targets used within a log-likelihood objective. We show how
using augmented outputs sampled according to edit distance improves a
maximum likelihood baseline by a considerable margin, on both machine
translation and speech recognition tasks. We believe this framework is
applicable to a wide range of probabilistic models with arbitrary
reward functions. In the future, we intend to explore the
applicability of this framework to other probabilistic models on tasks
with more complicated evaluation metrics.

\vspace*{.3cm}

%% file: ack.tex
\section{Acknowledgment}

We thank Dan Abolafia, Sergey Levine, Cinjon Resnick, Yujia Li, Ben
Poole, and the Google Brain team for insightful comments and
discussions.

%% file: new_appendix.tex
\section{Proofs}
\label{app:rml-proofs}

\setcounter{prop}{0}
\begin{proposition}
For any twice differentiable strictly convex closed potential $F$,
and $p, q\in\int(\domain)$:
\begin{align}
\df{q}{p}
&=
\df{p}{q}
+
\smallfrac{1}{4}
\trans{(p-q)}
\big(H_{F}(\altr)-H_F(\altaltr)\big) 
(p-q)
\label{eq:prop1-app}
\end{align}
for some $\altr=(1-\alpha)p+\alpha q$,~
($0\leq\alpha\leq \smallfrac{1}{2}$),~
$\altaltr=(1-\beta)q+\beta p$,~
($0\leq\beta\leq\smallfrac{1}{2}$).
\label{propn:taylor}
\end{proposition}

\begin{proof} %{Proposition \ref{propn:prelim}}
Let $f(p)$ denote $\nabla F(p)$ and consider the midpoint
$\smallfrac{q+p}{2}$.  One can express $F(\midr)$ by two Taylor
expansions around $p$ and $q$. By Taylor's theorem there is an $\altr
= (1-\alpha)p + \alpha q$ for $0\leq\alpha\leq \smallfrac{1}{2}$ and
$\altaltr = \beta p + (1-\beta)q$ for $0\leq\beta\leq\smallfrac{1}{2}$
such that
\begin{align}
F(\midr)
& ~=~
F(p) + (\midr-p)^\top f(p)
+ \smallfrac{1}{2}(\midr-p)^\top H_F(\altr)(\midr-p)
\\
& ~=~
F(q) + (\midr-q)^\top f(q)
+ \smallfrac{1}{2}(\midr-q)^\top H_F(\altaltr)(\midr-q)
,
\\
\mbox{hence, }
\qquad
2F(\midr)
& ~=~
2F(p) + (q-p)^\top f(p)
+ \smallfrac{1}{4}(q-p)^\top H_F(\altr)(q-p)
\\
& ~=~
2F(q) + (p-q)^\top f(q)
+ \smallfrac{1}{4}(p-q)^\top H_F(\altaltr)(p-q)
.
\end{align}
Therefore,
\begin{align}
	F(p)+F(q)-2F(\midr)
	& ~=~
	F(p) - F(q) - (p-q)^\top f(q)
	- \smallfrac{1}{4}(p-q)^\top H_F(\altaltr)(p-q)
	\\
	& ~=~
	F(q) - F(p) - (q-p)^\top f(p)
	- \smallfrac{1}{4}(q-p)^\top H_F(\altr)(q-p)
	\\
	& ~=~
        \df{p}{q}
	- \smallfrac{1}{4}(p-q)^\top H_F(\altaltr)(p-q)
        \label{eq:quad1}
	\\
	& ~=~
        \df{q}{p}
	- \smallfrac{1}{4}(q-p)^\top H_F(\altr)(q-p),
        \label{eq:quad2}
\end{align}
leading to the result.
\end{proof}

For the proof of Proposition~\ref{propn:variance}, we first need to
introduce a few definitions and background results.  A Bregman
divergence is defined from a strictly convex, differentiable, closed
potential function $F:\domain\rightarrow\Real$, whose strictly convex
conjugate $F^*:\range\rightarrow\Real$ is given by
$F^*(\targetr)=\sup_{\targetr\in\domain}\inner{\targetr}{\targetp}-F(\targetp)$
\cite{banerjeeetal05}.  Each of these potential functions have
corresponding transfers, $f:\domain\rightarrow\range$ and
$f^*:\range\rightarrow\domain$, given by the respective gradient maps
$f=\nabla F$ and $f^*=\nabla F^*$.  A key property is that
$f^*=f^{-1}$ \cite{banerjeeetal05}, which allows one to associate each
object $\targetp\in\domain$ with its transferred image
$\targetr=f(\targetp)\in\range$ and vice versa.
%each object $\targetp\in\range$ can be associated with its preimage
%$\targetr=f^*(\targetp)\in\domain$.
The main property of Bregman divergences we exploit
is that a divergence between any two domain objects
can always be equivalently expressed as a divergence 
between their transferred images;
that is, for any $\anyp\in\domain$ and $\targetp\in\domain$,
one has
\cite{banerjeeetal05}:
\begin{align}
\df{\anyp}{\targetp}
&=
\hskip2mmF(\anyp) 
-\inner{\anyp}{\targetr} 
+ F^*(\targetr)
=
\dfs{\targetr}{\anyr}
,
\label{eq:bregman dual forward}
\\
\df{\targetp}{\anyp}
&=
F^*(\anyr)
-\inner{\anyr}{\targetp} + 
F(\targetp) 
\hskip2mm
=
\dfs{\anyr}{\targetr}
,
\label{eq:bregman dual reverse}
\end{align}
where $\anyr\!=\!f(\anyp)$ and $\targetr\!=\!f(\targetp)$.  These
relations also hold if we instead chose $\anyr\!\in\!\range$ and
$\targetr\!\in\!\range$ in the range space, and used
$\anyp\!=\!\softmax(\anyr)$ and $\targetp\!=\!\softmax(\targetr)$.  In general
\eqref{eq:bregman dual forward} and \eqref{eq:bregman dual reverse}
are not equal.

Two special cases of the potential functions $F$ and $F^*$ are
interesting as they give rise to KL divergences. These two cases
include $F_\temp\left(p\right) = -\temp\ent{p}$ and $F^*_\temp(s) =
\temp \logsumexp{(s/\temp)} = \temp \log \sum_y \exp{(s(y)/\temp)}$,
where $\logsumexp(\cdot)$ denotes the log-sum-exp operator. The
respective gradient maps are $f_\temp(p) = \temp(\log(p) + \vec{1})$
and $\softmax_\temp(s) = \softmax(s / \temp) = {\small \frac{1}{\sum_y\!\exp(s(y) /
    \temp)}}\exp(s / \temp)$, where $\softmax_\temp$ denotes the
normalized exponential operator for $\frac{1}{\temp}$-scaled
logits. Below, we derive $\dfst{\targetr}{\anyr}$ for such
$F^*_\temp$:
\begin{equation}
\begin{aligned}
\dfst{s}{r}
& =~~ F^*_\temp(s) - F^*_\temp(r) - \trans{(s-r)} \nabla F^*_\temp(r)\\
& =~~ \temp \logsumexp(s / \temp) - \temp \logsumexp(r / \temp) - \trans{(s-r)} \softmax_\temp(r) \\
& =~~ -\temp \trans{\big( \left({s}/{\temp} - \logsumexp(s / \temp) \right) - \left({r}/{\temp} - \logsumexp(r / \temp)\right) \big)} \softmax_\temp(r) \\
& =~~ \temp \trans{\softmax_\temp(r)} \big( 
\left({r}/{\temp} - \logsumexp(r / \temp)\right) - \left({s}/{\temp} - \logsumexp(s / \temp) \right)
 \big)\\
& =~~ \temp \trans{\softmax_\temp(r)}\big( \log \softmax_\temp(r) - \log \softmax_\temp(s) \big)\\
& =~~ \temp\kl{\softmax_\temp(r)}{\softmax_\temp(s)}\\
& =~~ \temp\kl{q}{p}
\end{aligned}
\label{eq:tempered-kl}
\end{equation}

\begin{proposition}
The KL divergence between $p$ and $q$ in two directions can be
expressed as,
\begin{eqnarray}
\!\!\!\!\!\!\!\!
\kl{p}{q}
&\!\!\! = \!\!\!&
\kl{q}{p}
+
\smallfrac{1}{4\temp^2}\;
\mathrm{Var}_{y\sim \softmax(\altr / \temp)}\left[\anyr(y) - \targetr(y) \right]
-
\smallfrac{1}{4\temp^2}\;
\mathrm{Var}_{y\sim \softmax(\altaltr / \temp)}\left[\anyr(y) - \targetr(y) \right]
\label{eq:prop2-equality}
\\
&\!\!\! < \!\!\!&
\kl{q}{p}
+
\smallfrac{1}{\tau^2}\| \anyr-\targetr \|_2^2
,
\label{eq:prop2-inequality}
\end{eqnarray}

for some 
$\altr=(1-\alpha)\anyr+\alpha\targetr$,~
($0\leq\alpha\leq \smallfrac{1}{2}$),~
$\altaltr=(1-\beta)\targetr+ \beta \anyr$,~
($0\leq\beta\leq \smallfrac{1}{2}$).
\label{propn:variance}
\end{proposition}

\begin{proof}
First, for the potential function 
$F^*_\temp(\targetr) =\temp \logsumexp(\targetr/\temp)$
it is easy to verify that $F^*_\temp$ satisfies the conditions for
Proposition~\ref{propn:taylor}, and
\begin{equation}
H_{F^*_\temp}(\altr)=\smallfrac{1}{\temp} (\diag(\softmax_\temp(\altr))
-\softmax_\temp(\altr)\softmax_\temp(\altr)^\top)~,
\label{eq:F*hess}
\end{equation}
where $\diag(\vec{v})$ returns a square matrix the main diagonal of
which comprises a vector $\vec{v}$.
Therefore, by Proposition~\ref{propn:taylor} we obtain
\begin{equation}
\dfst{r}{s}
=
\dfst{s}{r}
+
\smallfrac{1}{4}\;
(s-r)^\top
(H_{F^*_\temp}(a)-H_{F^*_\temp}(b))
(s-r)
~,
\label{eq:prop2-1}
\end{equation}
for some $\altr=(1-\alpha)\anyr+\alpha\targetr$,~
($0\leq\alpha\leq \smallfrac{1}{2}$),~
$\altaltr=(1-\beta)\targetr+ \beta \anyr$,~
($0\leq\beta\leq \smallfrac{1}{2}$).
Note that by the specific form \eqref{eq:F*hess} we also have
\begin{align}
(\anyr-\targetr)^\top H_{F^*_\temp}(\altr)(\anyr-\targetr)
& =
\smallfrac{1}{\temp}
(\anyr-\targetr)^\top
\big(\diag(\softmax_\temp(\altr))
-\softmax_\temp(\altr)\softmax_\temp(\altr)^\top\big)
(\anyr-\targetr)
\label{eq:H}
\\
& =
\smallfrac{1}{\temp}
\big(
E_{\by\sim \softmax_\temp(\altr)}\left[(\anyr(\by)-\targetr(\by))^2\right]
-
E_{\by\sim \softmax_\temp(\altr)}\left[\anyr(\by)-\targetr(\by)\right]^2
\big)
\\
& =
\smallfrac{1}{\temp}
\mathrm{Var}_{\by\sim \softmax_\temp(\altr)}\left[\anyr(\by)-\targetr(\by)\right]\
,
\label{eq:var-hess1}
\\
\mbox{ and }
(\anyr-\targetr)^\top H_{F^*_\temp}(b)(\anyr-\targetr)
&=
\smallfrac{1}{\temp}
\mathrm{Var}_{\by\sim \softmax_\temp(b)}\left[\anyr(\by)-\targetr(\by)\right]\
\label{eq:var-hess2}
~.
\end{align}
Therefore, by combining \eqref{eq:var-hess1} and \eqref{eq:var-hess2} with 
\eqref{eq:prop2-1} we obtain
\begin{align}
\dfst{r}{s}
=
\dfst{s}{r}
+
\smallfrac{1}{4\temp}\;
\mathrm{Var}_{\by\sim \softmax_\temp(\altr)}\left[\anyr(\by)-\targetr(\by)\right]\
-
\smallfrac{1}{4\temp}\;
\mathrm{Var}_{\by\sim \softmax_\temp(\altaltr)}\left[\anyr(\by)-\targetr(\by)\right]\
~.
\label{eq:prop2-pre}
\end{align}
Equality \eqref{eq:prop2-equality} then follows by applying 
\eqref{eq:tempered-kl} to \eqref{eq:prop2-pre}.

Next, to prove the inequality in \eqref{eq:prop2-inequality},
let $\delta=\anyr-\targetr$ and observe that
\begin{align}
\dfst{r}{s}
-
\dfst{s}{r}
& =
\smallfrac{1}{4}
\delta^\top \big(H_{F^*_\temp}(\altr)-H_{F^*_\temp}(\altaltr)\big) \delta
\\
& = 
\smallfrac{1}{4\temp}
\delta^\top\diag(\softmax_\temp(\altr)-\softmax_\temp(\altaltr))\delta
+
\smallfrac{1}{4\temp}
\big(\delta^\top \softmax_\temp(\altaltr)\big)^2
-
\smallfrac{1}{4\temp}
\big(\delta^\top \softmax_\temp(\altr)\big)^2
\\
& \leq
\smallfrac{1}{4\temp}
\|\delta\|_2^2\|\softmax_\temp(\altr)-\softmax_\temp(\altaltr)\|_\infty
+
\smallfrac{1}{4\temp}
\|\delta\|_2^2\|\softmax_\temp(\altaltr)\|_2^2
\\
& \leq
\smallfrac{1}{2\temp}
\|\delta\|_2^2
+
\smallfrac{1}{4\temp}
\|\delta\|_2^2
\label{eq:prop2-fin}
\end{align}
since
$\|\softmax_\temp(\altr)-\softmax_\temp(\altaltr)\|_\infty\leq2$
and
$\|\softmax_\temp(\altaltr)\|_2^2\leq\|\softmax_\temp(\altaltr)\|_1^2\leq1$.
The result follows by applying \eqref{eq:tempered-kl} to \eqref{eq:prop2-fin}.
\end{proof}

%% file: paper.bbl
\begin{thebibliography}{10}

\bibitem{andor2016globally}
D.~Andor, C.~Alberti, D.~Weiss, A.~Severyn, A.~Presta, K.~Ganchev, S.~Petrov,
  and M.~Collins.
\newblock Globally normalized transition-based neural networks.
\newblock {\em arXiv:1603.06042}, 2016.

\bibitem{bahdanau2016actor}
D.~Bahdanau, P.~Brakel, K.~Xu, A.~Goyal, R.~Lowe, J.~Pineau, A.~Courville, and
  Y.~Bengio.
\newblock An actor-critic algorithm for sequence prediction.
\newblock {\em arXiv:1607.07086}, 2016.

\bibitem{bahdanau2014neural}
D.~Bahdanau, K.~Cho, and Y.~Bengio.
\newblock Neural machine translation by jointly learning to align and
  translate.
\newblock {\em ICLR}, 2015.

\bibitem{banerjeeetal05}
A.~Banerjee, S.~Merugu, I.~S. Dhillon, and J.~Ghosh.
\newblock Clustering with {B}regman divergences.
\newblock {\em JMLR}, 2005.

\bibitem{bengioetal15}
S.~Bengio, O.~Vinyals, N.~Jaitly, and N.~M. Shazeer.
\newblock Scheduled sampling for sequence prediction with recurrent neural
  networks.
\newblock {\em NIPS}, 2015.

\bibitem{chan2015listen}
W.~Chan, N.~Jaitly, Q.~V. Le, and O.~Vinyals.
\newblock Listen, attend and spell.
\newblock {\em ICASSP}, 2016.

\bibitem{cho2014learning}
K.~Cho, B.~Van~Merri{\"e}nboer, C.~Gulcehre, D.~Bahdanau, F.~Bougares,
  H.~Schwenk, and Y.~Bengio.
\newblock Learning phrase representations using rnn encoder-decoder for
  statistical machine translation.
\newblock {\em EMNLP}, 2014.

\bibitem{chorowski2014end}
J.~Chorowski, D.~Bahdanau, K.~Cho, and Y.~Bengio.
\newblock End-to-end continuous speech recognition using attention-based
  recurrent nn: first results.
\newblock {\em arXiv:1412.1602}, 2014.

\bibitem{chorowski15}
J.~K. Chorowski, D.~Bahdanau, D.~Serdyuk, K.~Cho, and Y.~Bengio.
\newblock Attention-based models for speech recognition.
\newblock {\em NIPS}, 2015.

\bibitem{daumeetal09}
H.~Daum{\'e}, III, J.~Langford, and D.~Marcu.
\newblock Search-based structured prediction.
\newblock {\em Mach. Learn. J.}, 2009.

\bibitem{degris2012model}
T.~Degris, P.~M. Pilarski, and R.~S. Sutton.
\newblock Model-free reinforcement learning with continuous action in practice.
\newblock {\em ACC}, 2012.

\bibitem{domke12}
J.~Domke.
\newblock Generic methods for optimization-based modeling.
\newblock {\em AISTATS}, 2012.

\bibitem{gimpel2010softmax}
K.~Gimpel and N.~A. Smith.
\newblock Softmax-margin crfs: Training log-linear models with cost functions.
\newblock {\em NAACL}, 2010.

\bibitem{hintonetal15}
G.~{Hinton}, O.~{Vinyals}, and J.~{Dean}.
\newblock Distilling the knowledge in a neural network.
\newblock {\em arXiv:1503.02531}, 2015.

\bibitem{hochreiter1997long}
S.~Hochreiter and J.~Schmidhuber.
\newblock Long short-term memory.
\newblock {\em Neural Computation}, 1997.

\bibitem{kappen12}
H.~J. Kappen, V.~G{\'o}mez, and M.~Opper.
\newblock Optimal control as a graphical model inference problem.
\newblock {\em Mach. Learn. J.}, 2012.

\bibitem{kimetal14}
B.~Kim, A.~M. Farahmand, J.~Pineau, and D.~Precup.
\newblock Learning from limited demonstrations.
\newblock {\em NIPS}, 2013.

\bibitem{kumaretal15}
A.~Kumar, O.~Irsoy, J.~Su, J.~Bradbury, R.~English, B.~Pierce, P.~Ondruska,
  I.~Gulrajani, and R.~Socher.
\newblock Ask me anything: Dynamic memory networks for natural language
  processing.
\newblock {\em ICML}, 2016.

\bibitem{laffertyetal01}
J.~D. Lafferty, A.~McCallum, and F.~C.~N. Pereira.
\newblock {Conditional Random Fields}: Probabilistic models for segmenting and
  labeling sequence data.
\newblock {\em ICML}, 2001.

\bibitem{lecunetal98}
Y.~LeCun, L.~Bottou, Y.~Bengio, and P.~Haffner.
\newblock Gradient based learning applied to document recognition.
\newblock {\em Proceedings of {IEEE}}, 1998.

\bibitem{levinekoltun13a}
S.~Levine and V.~Koltun.
\newblock Guided policy search.
\newblock {\em ICML}, 2013.

\bibitem{levinekoltun13b}
S.~Levine and V.~Koltun.
\newblock Variational policy search via trajectory optimization.
\newblock {\em NIPS}, 2013.

\bibitem{lopezpazetal16}
D.~Lopez-Paz, B.~Sch{\"o}lkopf, L.~Bottou, and V.~Vapnik.
\newblock Unifying distillation and privileged information.
\newblock {\em ICLR}, 2016.

\bibitem{luong2015effective}
M.-T. Luong, H.~Pham, and C.~D. Manning.
\newblock Effective approaches to attention-based neural machine translation.
\newblock {\em EMNLP}, 2015.

\bibitem{luong2015}
M.-T. Luong, I.~Sutskever, Q.~V. Le, O.~Vinyals, and W.~Zaremba.
\newblock Addressing the rare word problem in neural machine translation.
\newblock {\em ACL}, 2015.

\bibitem{mnih2016asynchronous}
V.~Mnih, A.~P. Badia, M.~Mirza, A.~Graves, T.~P. Lillicrap, T.~Harley,
  D.~Silver, and K.~Kavukcuoglu.
\newblock Asynchronous methods for deep reinforcement learning.
\newblock {\em ICML}, 2016.

\bibitem{petersetal10}
J.~Peters, K.~M{\"u}lling, and Y.~Alt{\"u}n.
\newblock Relative entropy policy search.
\newblock {\em AAAI}, 2010.

\bibitem{povey2011kaldi}
D.~Povey, A.~Ghoshal, G.~Boulianne, et~al.
\newblock The kaldi speech recognition toolkit.
\newblock {\em ASRU}, 2011.

\bibitem{ranzatoetal15}
M.~Ranzato, S.~Chopra, M.~Auli, and W.~Zaremba.
\newblock Sequence level training with recurrent neural networks.
\newblock {\em ICLR}, 2016.

\bibitem{rossetal10}
S.~Ross, G.~J. Gordon, and J.~A. Bagnell.
\newblock {A Reduction of Imitation Learning and Structured Prediction to
  No-Regret Online Learning}.
\newblock {\em AISTATS}, 2010.

\bibitem{shen2015minimum}
S.~Shen, Y.~Cheng, Z.~He, W.~He, H.~Wu, M.~Sun, and Y.~Liu.
\newblock Minimum risk training for neural machine translation.
\newblock {\em ACL}, 2016.

\bibitem{silveretal16}
D.~Silver et~al.
\newblock Mastering the game of {Go} with deep neural networks and tree search.
\newblock {\em Nature}, 2016.

\bibitem{stoyanovetal11}
V.~Stoyanov, A.~Ropson, and J.~Eisner.
\newblock Empirical risk minimization of graphical model parameters given
  approximate inference, decoding, and model structure.
\newblock {\em AISTATS}, 2011.

\bibitem{strubelletal15}
E.~Strubell, L.~Vilnis, K.~Silverstein, and A.~McCallum.
\newblock Learning dynamic feature selection for fast sequential prediction.
\newblock {\em ACL}, 2015.

\bibitem{sutskeveretal14}
I.~Sutskever, O.~Vinyals, and Q.~V. Le.
\newblock Sequence to sequence learning with neural networks.
\newblock {\em NIPS}, 2014.

\bibitem{suttonbarto98}
R.~S. Sutton and A.~G. Barto.
\newblock {\em Reinforcement Learning: An Introduction}.
\newblock MIT Press, 1998.

\bibitem{suttonetal00}
R.~S. Sutton, D.~A. McAllester, S.~P. Singh, and Y.~Mansour.
\newblock Policy gradient methods for reinforcement learning with function
  approximation.
\newblock {\em NIPS}, 2000.

\bibitem{taskar2004}
B.~Taskar, C.~Guestrin, and D.~Koller.
\newblock Max-margin markov networks.
\newblock {\em NIPS}, 2004.

\bibitem{todorov06}
E.~Todorov.
\newblock Linearly-solvable markov decision problems.
\newblock {\em NIPS}, 2006.

\bibitem{tsochantaridisetal05}
I.~Tsochantaridis, T.~Joachims, T.~Hofmann, and Y.~Altun.
\newblock Large margin methods for structured and interdependent output
  variables.
\newblock {\em JMLR}, 2005.

\bibitem{vanhasseltetal15}
H.~van Hasselt, A.~Guez, and D.~Silver.
\newblock Deep reinforcement learning with double q-learning.
\newblock {\em arXiv:1509.06461}, 2015.

\bibitem{vapnikizmailov15}
V.~Vapnik and R.~Izmailov.
\newblock Learning using privileged information: Similarity control and
  knowledge transfer.
\newblock {\em JMLR}, 2015.

\bibitem{vlassisetal09}
N.~Vlassis, M.~Toussaint, G.~Kontes, and S.~Piperidis.
\newblock Learning model-free robot control by a {Monte Carlo EM} algorithm.
\newblock {\em Autonomous Robots}, 2009.

\bibitem{volkovsetal11}
M.~{Volkovs}, H.~{Larochelle}, and R.~{Zemel}.
\newblock Loss-sensitive training of probabilistic conditional random fields.
\newblock {\em arXiv:1107.1805v1}, 2011.

\bibitem{williams92}
R.~J. Williams.
\newblock Simple statistical gradient-following algorithms for connectionist
  reinforcement learning.
\newblock {\em Mach. Learn. J.}, 1992.

\bibitem{williams1991function}
R.~J. Williams and J.~Peng.
\newblock Function optimization using connectionist reinforcement learning
  algorithms.
\newblock {\em Connection Science}, 1991.

\bibitem{wiseman2016sequence}
S.~Wiseman and A.~M. Rush.
\newblock Sequence-to-sequence learning as beam-search optimization.
\newblock {\em arXiv:1606.02960}, 2016.

\end{thebibliography}
